\documentclass{./cls/svproc}
\usepackage{geometry}
\usepackage[T1]{fontenc}

\usepackage{url}

\usepackage{enumitem}
\usepackage[usenames,dvipsnames]{xcolor}

\usepackage{amsmath,amssymb,amsfonts,amsthm,dsfont}

\usepackage{algorithm,algorithmicx,listings}
\usepackage[noend]{algpseudocode}
\algnewcommand{\LineComment}[1]{{\Statex \vspace{1ex} \(\triangleright\) #1}}

\usepackage{graphicx}
\usepackage[]{draftfigure}
\usepackage{todonotes}
\usepackage{subfig}

\newcommand{\bfb}{\mathbf{b}}
\newcommand{\bfc}{\mathbf{c}}

\newcommand{\bfe}{\mathbf{e}}

\newcommand{\bfq}{\mathbf{q}}

\newcommand{\bfu}{\mathbf{u}}

\newcommand{\bfx}{\mathbf{x}}
\newcommand{\bfy}{\mathbf{y}}

\newcommand{\bbN}{\mathbb{N}}

\newcommand{\bbR}{\mathbb{R}}

\newcommand{\calB}{\mathcal{B}}

\newcommand{\calE}{\mathcal{E}}

\newcommand{\calP}{\mathcal{P}}
\newcommand{\calQ}{\mathcal{Q}}

\newcommand{\calX}{\mathcal{X}}

\newcommand{\sbfont}[1]{\text{#1}}
\newcommand{\snear}{\sbfont{near}}
\newcommand{\srand}{\sbfont{rand}}
\newcommand{\snew}{\sbfont{new}}
\newcommand{\smin}{\sbfont{min}}
\newcommand{\soverlap}{\sbfont{overlap}}
\newcommand{\sexplore}{\sbfont{explore}}
\newcommand{\scurrent}{\sbfont{current}}
\newcommand{\ssample}{\sbfont{sample}}

\newcommand{\BL}[1]{}
\newcommand{\removed}[1]{}

\theoremstyle{definition}

\begin{document}
\mainmatter              
\title{Safe Bubble Cover for Motion Planning on Distance Fields}
\titlerunning{Safe Bubble Cover for Motion Planning on Distance Fields}  
%
\author{Ki Myung Brian Lee\inst{1} \and
Zhirui Dai\inst{1}\and
Cedric Le Gentil\inst{2} \and
Lan Wu\inst{2} \and
Nikolay~Atanasov\inst{1}  \and
Teresa Vidal-Calleja\inst{2}
}
\authorrunning{K. M. B. Lee et al.} 
%
\tocauthor{Ki Myung Brian Lee, Zhirui Dai, Cedric Le Gentil, Lan Wu, Nikolay Atanasov and Teresa Vidal-Calleja}
%

\institute{University of California San Diego, La Jolla, CA 92093, USA \and
University of Technology Sydney, Ultimo NSW 2007, Australia
}

\maketitle              

\setcounter{footnote}{0}
\begin{abstract}
We consider the problem of planning collision-free trajectories on distance fields.
Our key observation is that querying a distance field at one configuration reveals a region of safe space whose radius is given by the distance value, obviating the need for additional collision checking within the safe region.
We refer to such regions as safe bubbles, and show that safe bubbles can be obtained from any Lipschitz-continuous safety constraint.
Inspired by sampling-based planning algorithms, we present three algorithms for constructing a safe bubble cover of free space, named bubble roadmap (BRM), rapidly exploring bubble graph (RBG), and expansive bubble graph (EBG).
The bubble sampling algorithms are combined with a hierarchical planning method that first computes a discrete path of bubbles, followed by a continuous path within the bubbles computed via convex optimization. Experimental results show that the bubble-based methods yield up to 5-10 times cost reduction relative to conventional baselines while simultaneously reducing computational efforts by orders of magnitude. 
\end{abstract}
%
%
%


\section{Introduction}
\label{sec:introduction}



Motion planning is a foundational component of robot autonomy.
It is important not only for operating in complex environments, but also as robots become increasingly physically capable.
To fully harness the physical capabilities of modern robots in complex environments, it is necessary to develop planning algorithms that rapidly produce safe and dynamically feasible trajectories.


A significant bottleneck in conventional motion planning algorithms is collision checking.
To ensure collision avoidance, a planning algorithm typically samples along a candidate trajectory to check for a potential collision.
Although it is possible to accelerate collision queries~\cite{ramsey2024capt}, the computational efficiency of planning algorithms remains bounded by the number of collision checking queries.


We propose an approach to circumvent this fundamental limitation by sampling continuous regions of safe space instead of collision checking individual configurations.
Our key insight is that querying a distance field representation of the environment yields the radius of collision-free space around the query point, since it represents the distance to the closest obstacle. See Fig.~\ref{fig:teaser} for an illustration. We define such regions as `safe bubbles', and show that they can be derived from any Lipschitz-continuous safety constraint, distance fields being a special case for which many perception algorithms are available~\cite{wu2023loggpis,coiffier2024lipschitzndf,oritz2022isdf}.
\begin{figure}
    \centering
    \includegraphics[width=\linewidth]{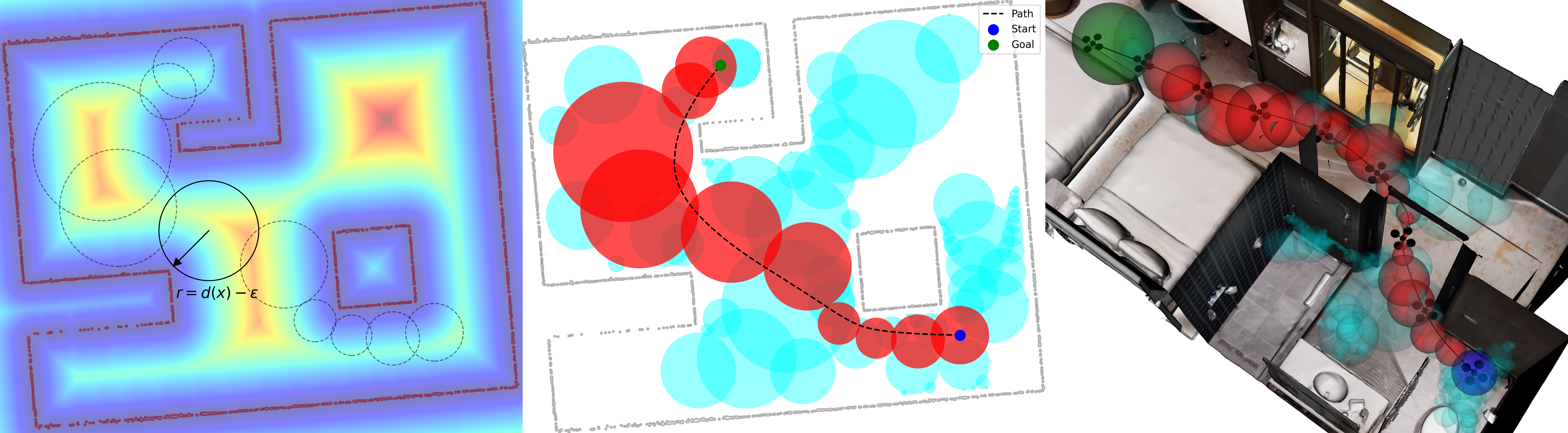}
    \caption{Left: We construct `safe bubbles' from a distance field representation of the environment, whose radii are given by the distance to obstacles. Middle: We present three algorithms for sampling safe bubbles (cyan), and a hierarchical planning method that first computes a bubble path (red) and then a continuous trajectory (dashed line) within the bubble path. Right: Our approach scales effectively to higher dimensions.}
    \label{fig:teaser}
\end{figure}

To sample safe bubbles from a distance field representation, we first introduce three sampling algorithms, named bubble roadmap (BRM), rapidly-exploring bubble graph (RBG), and expansive bubble graph (EBG), inspired by conventional sampling-based planning algorithms.
Given a bubble cover of free space, we present an efficient hierarchical planning method that first computes a sequence of intersecting bubbles minimizing an upper bound on the trajectory length, followed by continuous planning of a dynamically feasible trajectory via convex optimization.

We evaluate the practical utility of safe bubble planning through comparisons against popular sampling-based planning algorithms, PRM$^*$ and RRT$^*$. 
The results show that safe bubble planning provides orders of magnitude improvements in computational efficiency along with 4-5 times improvement in trajectory cost.
We also compare the three algorithms and find that the bubble-based methods can benefit from techniques for even spatial coverage in classical sampling-based planning.  We view the main contribution of this work as establishing the foundation for a new class of planning algorithms for continuous implicit environment representations.

\section{Related Work}
\label{sec:relatedwork}

Sampling-based motion planning methods plan collision-free paths by sampling safe configurations and connecting them with collision-free edges.
Probabilistic roadmap (PRM) variants \cite{kavraki1996prm,bohlin2000lazyprm,karaman2011optimal} draw samples from a uniform distribution, whereas rapidly exploring random tree (RRT) variants \cite{lavalle2001rrt,karaman2011optimal} promote even spatial coverage by "steering" the next sample toward a uniform random sample.
For dynamic feasibility, a continuous trajectory optimization step is often employed to smooth the resulting path \cite{zucker2013chomp}.
An in-depth treatise of discrete and continuous planning methods can be found in \cite{lavalle_planning_book}.
Our approach unifies continuous and discrete planning by sampling continuous regions of collision-free configurations, rather than individual configurations.

Although sampling-based methods are powerful for planning in complex, high-dimensional spaces, computation is often hindered by the cost of repeated \emph{collision checking}.
To this end, previous work attempts to reduce either the number \cite{bohlin2000lazyprm,hou2020posterior} or computation time \cite{ramsey2024capt} of collision checks. We present an orthogonal approach that replaces collision checking with \emph{safe space} construction.

The idea of safe space construction is popular especially for quadrotors \cite{liu2017sfc,chen2024splatnav,wu2024optimal}.
These methods generate a convex partition of free space, using polytopes \cite{liu2017sfc,chen2024splatnav,wu2024optimal} or ellipsoids \cite{huh2022probabilisticellipsoid}.
Given such partitions, convex optimization methods \cite{liu2017sfc,maruci2023gcs_planning,chen2024splatnav,wu2024optimal} can compute cost-optimal continuous-space trajectories. However, constructing a convex partition of free space is not readily compatible with environment representations obtained by perception and mapping methods. The locations and shapes of polytopes or ellipsoids are typically optimized within an occupancy grid or point cloud map of the environment.


In the perception community, an emerging trend is to use implicit representations of occupied space, in the form of distance fields~\cite{coiffier2024lipschitzndf,ueda2022nddf,wu2023loggpis,legentil2024accurate,wu2024vdb} or radiance fields~\cite{mildenhall2021nerf,kerbl2023gsplat}, which allow efficient and accurate reconstruction of an environment. Previous work on planning in these representations explored differentiating through a scene representation for nonlinear trajectory optimization~\cite{adamkiewicz2022nerfplanning,wu2023loggpis,zucker2013chomp} or building a safe convex polytope from Gaussian splats~\cite{chen2024splatnav}.
Our approach generates a safe convex cover by simply querying a distance field representation~\cite{wu2023loggpis}, and shows that any Lipschitz-continuous representation~\cite{revay2020lipschitz,coiffier2024lipschitzndf} can be used. 


Of particular relevance to our work are~\cite{musil2022spheremap,kim2018dancing,noel2024skeleton,ren2022bubbleplanner}.
These approaches also consider spherical safe space representations, computed by keeping track of obstacle points~\cite{kim2018dancing,musil2022spheremap,ren2022bubbleplanner} or a distance transformation of occupancy grid~\cite{noel2024skeleton}, coupled with specialized sampling strategies.
We present a theoretical foundation for these "safe bubble" approaches, and explore their generality as a new class of sampling-based planning algorithms for implicit representations.

\section{Problem Formulation}
\label{sec:problem}

Consider a robot described by a nonlinear dynamical system:
\begin{equation} \label{eq:nonlinear_system}
  \dot{\bfx}(t) = f(\bfx(t), \bfu(t)),
\end{equation}
where $\bfx(t) \in \calX \subseteq \bbR^n$ is the state and $\bfu(t) \in \bbR^m$ is the control input. We assume that \eqref{eq:nonlinear_system} is differentially flat.

\begin{definition}[{\cite[Ch.~2]{rigatos2015nonlinear}}]
    A nonlinear dynamical system \eqref{eq:nonlinear_system} is \emph{differentially flat} if there exists a \emph{flat output} $\bfy(t) \in \bbR^m$ such that:
    \begin{itemize}[nosep]
        \item $\bfy(t)$ is expressed by a smooth function of $\bfx(t)$, $\bfu(t)$, and the first $r \in \bbN$ derivatives of $\bfu(t)$:
        \[
        \bfy(t) = h(\bfx(t),\bfu(t),\dot{\bfu}(t),\ldots,\bfu^{(r)}(t)),
        \]

        \item $\bfx(t)$ and $\bfu(t)$ can be expressed as smooth functions of $\bfy(t)$ and its derivatives:
        \[
        \begin{aligned}
            \bfx(t) &= \alpha(\bfy(t),\dot{\bfy}(t),\ldots,\bfy^{(r-1)}(t)),\\
            \bfu(t) &= \beta(\bfy(t),\dot{\bfy}(t),\ldots,\bfy^{(r)}(t)),          
        \end{aligned}
        \]

        \item there exists no differential relation among the output derivatives such as $\eta(\bfy,\dot{\bfy},\ldots) = 0$.
        
    \end{itemize}
\end{definition}

Many robot systems are differentially flat, including quadrotors \cite{morrell2018differential,mellinger2011differential}, fully actuated Lagrangian systems~\cite{murray1995differential}, and some non-holonomic systems \cite{murray1995differential}.

Our objective is to plan a sufficiently smooth output trajectory $\bfy(t)$ minimizing a convex cost function $c : \bbR^m \rightarrow \bbR$ subject to constraints $\bfy(t) \notin \Omega \subset \bbR^m$. We assume that the distance function of the unsafe set $\Omega$ is known:
\begin{equation}\label{eq:sdf_def}
    d_\Omega(\bfy) = \min_{\bfq \in \partial \Omega} \| \bfy - \bfq \|,
\end{equation}
which maybe constructed online from sensor data~\cite{wu2023loggpis,oleynikova2017voxblox,oritz2022isdf}. 
The problem of planning a collision-free output trajectory can then be formulated as follows:
\begin{equation}\label{eq:problem}
\begin{split}
    \min_{\bfy \in C^r} \quad \int_{0}^{T} & c(\bfy(t)) dt, \\
    \text{s.t. (start and goal point)} \quad & \bfy(0) = \bfy_{s},\; \bfy(T) = \bfy_{g}, \\
    \text{ (collision avoidance)} \quad& d_\Omega(\bfy(t)) \geq \epsilon, \quad \forall t \in [0, T],
\end{split}
\end{equation}
where $\bfy_{s}$ and $\bfy_{g}$ are given start and goal, $\epsilon > 0$ is a parameter, $T$ is the planning horizon, and $\bfy \in C^r$ means that $\bfy: [0,T] \rightarrow \bbR^m$ has a continuous $r$-th derivative.

\section{Constructing a Bubble Cover}\label{sec:sampling}

We present algorithms for efficiently approximating the collision-avoidance constraint in~\eqref{eq:problem} by sampling a \emph{safe bubble cover}.
We define and explain \emph{safe bubbles} in Sec.~\ref{sec:safe_bubbles}, and present three randomized algorithms for constructing such bubbles inspired by well-known sampling-based planning algorithms in Sec.~\ref{sec:brm},~\ref{sec:rbt},~\ref{sec:ebt}.

\subsection{Bubbles for Safe Space Representation}\label{sec:safe_bubbles}
Our core insight for solving \eqref{eq:problem} efficiently is that the distance field $d_\Omega$ can be used to build a set of \emph{bubbles} covering the safe space.
Intuitively, since the distance field $d_\Omega(\bfy)$ yields the closest distance to the unsafe set $\Omega$, there can be no obstacle within radius $d_\Omega(\bfy)$ of a query point $\bfy$. 
Otherwise, $d_\Omega(\bfy)$ must be smaller to reflect the presence of an obstacle within the ball.

More generally, we can define such bubbles whenever the safety constraint in \eqref{eq:problem} is specified by a \emph{Lipschitz continuous} function, which may be constructed from data using methods such as \cite{revay2020lipschitz,coiffier2024lipschitzndf}. We define a \emph{safe bubble} and discuss its properties next.

\begin{theorem}[Safe bubble]\label{thm:bubble}
Consider a constraint $l(\bfy) \geq 0$, where $l:\bbR^m \rightarrow \bbR$ is Lipschitz continuous with constant $L$. Then, for any $\bfy$ such that $l(\bfy) \geq 0$, all points $\bfy'$ in ball $B(\bfy, \frac{l(\bfy)}{L})$ centered at $\bfy$ with radius $\frac{l(\bfy)}{L}$ also satisfy $l(\bfy') \geq 0$.
\end{theorem}

\begin{proof}
With Lipschitz continuity, we have $l(\bfy) - l(\bfy') \leq L\|\bfy - \bfy'\|$ for any $\bfy,\bfy' \in \bbR^m$. For $\bfy' \in B(\bfy, \frac{l(\bfy)}{L})$, we have $\|\bfy - \bfy'\| \leq \frac{l(\bfy)}{L}$. Combining the two inequalities yields $l(\bfy') \geq 0$ (i.e., arbitrary points $\bfy'$ in $B(\bfy, \frac{l(\bfy)}{L})$ are feasible). 
\end{proof}





Theorem~\ref{thm:bubble} includes the distance field constraint in \eqref{eq:problem} as a special case with $l(\bfy) = d_{\Omega}(\bfy) - \epsilon$ because the distance function $d_\Omega$ of any set $\Omega$ is Lipschitz continuous with constant $L=1$.
Importantly, Theorem~\ref{thm:bubble} implies that querying the distance field at a point where $d_\Omega(\bfy) \geq \epsilon$ readily yields a safe bubble around the point.
Then, an important consideration is \emph{how to sample} such query points at which to generate safe bubbles to cover the safe space. Next, we take inspiration from sampling-based planning methods to generate safe bubbles.



\subsection{Bubble Roadmap}\label{sec:brm}
\begin{algorithm}[b]
    \caption{Bubble Roadmap Algorithm}\label{alg:brm}
    \begin{algorithmic}[1]
    \Statex \textbf{Parameters:} No. of samples $N_\ssample$, minimum radius $r_{\text{min}}$, footprint radius $\epsilon$. 
    \State $\mathcal{C} \leftarrow \textsc{sample}(N_\ssample)$ \Comment{Sample $N$ random centers.}\label{alg:brm:sample}
    \State $\mathcal{B} = \{ \}$ \Comment{Initialize bubble cover as an empty set.}
    \For{ $\mathbf{y} \in \mathcal{C}$ } 
    \If{ $d_\Omega(\bfy) - \epsilon > r_{\text{min}}$ } \Comment{If radius greater than $r_{\text{min}}$,}\label{alg:brm:size_check}
    \State $\mathcal{B} \leftarrow \mathcal{B} \cup \{B_{\text{new}} = (\bfy, d_\Omega(\bfy) - \epsilon)\}$ \Comment{add to set of bubbles.}\label{alg:brm:keep}
    \EndIf
    \EndFor
    \Return $\mathcal{B}$
    \end{algorithmic}
\end{algorithm}
\begin{figure}
    \centering
    \subfloat[100 samples]{
    \includegraphics[width=0.3\textwidth]{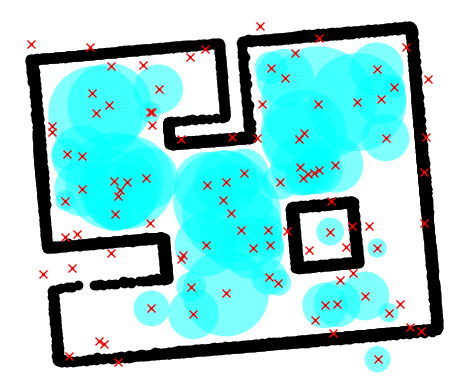}
    }%
    \subfloat[300 samples]{
    \includegraphics[width=0.3\textwidth]{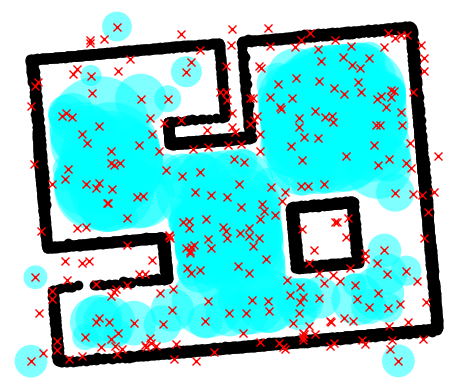}
    }%
    \subfloat[1000 samples]{
    \includegraphics[width=0.3\textwidth]{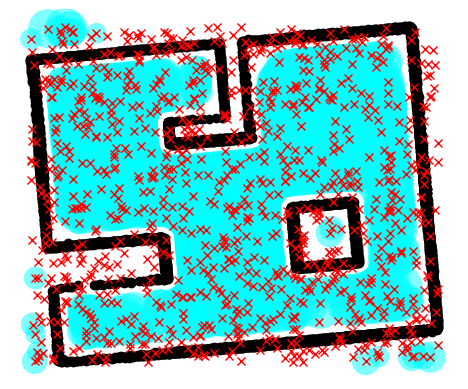}
    }%
    \caption{Illustration of the BRM algorithm with varying number of samples (red crosses) and bubbles (cyan). Not all random samples lead to a valid bubble due to footprint and minimum radius requirements. The free space is filled as the number of samples increases (from (a) to (c)).}
    \label{fig:brm_illustration}
\end{figure}
Inspired by PRM~\cite{kavraki1996prm}, the simplest method we propose is the \emph{bubble roadmap} (BRM), described in Alg~\ref{alg:brm}. Just as PRM samples random configurations, BRM samples random bubble centers from a uniform distribution.  
In doing so, a minimum radius requirement is enforced, which speeds up subsequent planning inside the bubble cover, discussed in Sec.~\ref{sec:planning}.  
This is achieved by first sampling possible centers for the bubbles, computing the radii as per Theorem~\ref{thm:bubble}, and keeping only those large enough. 
Results with varying numbers of samples are visualized in Fig.~\ref{fig:brm_illustration}.

A potential concern with BRM is that some samples may be redundant because some bubbles may contain other. However, we show that the event of one bubble containing another is of probability zero.

\begin{theorem}[Random bubbles do not contain each other]\label{thm:bubble_containment}
Let $\bfy_{1}$, $\bfy_{2}$ be samples from a distribution such that $\|\bfy_1 - \bfy_2\| = \delta$ is probability zero $\forall \delta \geq 0$. Let $B_{1} = (\bfy_{1}, r_1)$ and $B_{2} = (\bfy_{2}, r_2)$ be safe bubbles for a Lipschitz constraint $l(\bfy) \geq 0$ as per Theorem~\ref{thm:bubble}. Then, the probability of $B_1 \subseteq B_2$ is zero.
\end{theorem}

\begin{proof}
First, notice $B_{1} \subseteq B_{2}$ iff $\|\bfy_1 - \bfy_2\| \leq |r_1 - r_2|$.
With Theorem~\ref{thm:bubble}, we have $|r_1 - r_2| = |\tfrac{l(\bfy_{1}) - l(\bfy_2)}{L}|$.
Meanwhile, by Lipschitz continuity, $|\tfrac{l(\bfy_{1}) - l(\bfy_2)}{L}| \leq \| \bfy_1 - \bfy_2\|$.
Thus, $B_1 \subseteq B_2$ is only possible when $\|\bfy_1 - \bfy_2\| = |r_1 - r_2|$, which is probability zero as per assumption.
\end{proof}

Theorem~\ref{thm:bubble_containment} applies to most practical sampling distributions, including uniform. However, the problem of redundancy remains, because Theorem~\ref{thm:bubble_containment} only applies to the case of one bubble wholly containing another. 
It is still possible and empirically frequent that the union of multiple bubbles contains another bubble. Next, we consider sampling strategies that improve on BRM in terms of overlap redundancy.



\subsection{Rapidly Exploring Bubble Graph}\label{sec:rbt}
%
Although BRM samples centers from a uniform distribution, the resulting bubbles can exhibit uneven coverage of space because their radii vary over space. Inspired by RRT~\cite{lavalle2001rrt}, we propose the \emph{rapidly exploring bubble graph} (RBG) algorithm that improves coverage by expanding bubbles toward random samples. 

\begin{algorithm}[t]
    \caption{Rapidly-exploring Bubble Graph}\label{alg:rbt}
    \begin{algorithmic}[1]
    \Statex \textbf{Parameters}: Starting location $\bfy_{s}$, minimum radius $r_{\smin}$, footprint radius $\epsilon$
    \State $\mathcal{B} \leftarrow \{B(\mathbf{y}_{s}, d_\Omega(\mathbf{y}_s) - \epsilon)\}$ \Comment{Initialize with bubble at starting location}\label{alg:rbt:init}
    \While{termination condition is not met}
    \State $\bfy_{\srand} \leftarrow \textsc{sample\_outside}(\mathcal{B})$ \Comment{Sample outside current bubbles}\label{alg:rbt:sample}
    \State $B_{\snear} = (\bfy_{\snear}, r_{\snear}) \leftarrow \min_{B \in \mathcal{B}} d(\bfy_{\srand}, B)$ \Comment{Find nearest bubble (see \eqref{eq:dist_bubble})}\label{alg:rbt:nearest}
    \State $\bfy_{\snew} \leftarrow \textsc{steer}(\bfy_\srand, B_\snear)$ \Comment{Steer to perimeter of nearest bubble (see \eqref{eq:rbt_steer})}\label{alg:rbt:steer}
    \If{$r_{\snew} = d(\bfy_{\snew}) - \epsilon > r_{\smin} $}
    \State $\mathcal{B} \leftarrow \mathcal{B} \cup \{B(\bfy_{\snew}, r_{\snew})\}$ \Comment{Add new bubble if size requirement holds}
    \EndIf
    \EndWhile    
    \end{algorithmic}
\end{algorithm}
\begin{figure}
    \centering
    \subfloat[]{
    \includegraphics[width=0.32\textwidth]{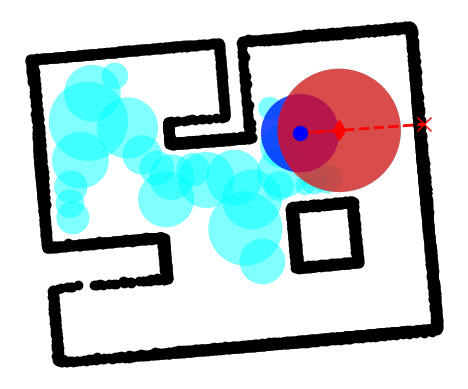}
    }%
    \subfloat[]{
    \includegraphics[width=0.32\textwidth]{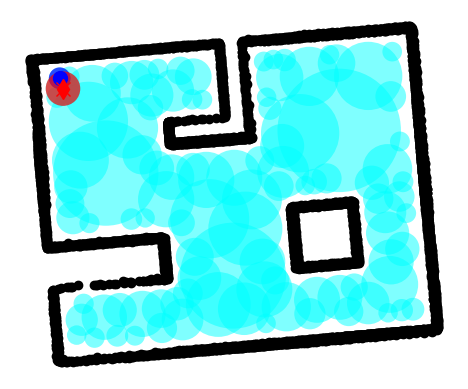}
    }
    \caption{Illustration of the RBG algorithm. Similar to RRT, RBG promotes even spatial coverage by sampling and steering towards random points (red crosses). The center of a new bubble (red diamond) is set at the perimeter of the nearest bubble (blue). Repeating this process from a) to b), we obtain a safe bubble cover of the free space with limited variation in radii.}
    \label{fig:rbg_illustration}
    \vspace{-2ex}
\end{figure}

The RBG algorithm is described in Alg.~\ref{alg:rbt}, and illustrated in Fig.~\ref{fig:rbg_illustration}. A set of bubbles is initialized with a bubble at a seed point $\bfy_{s}$, which may be the start point in \eqref{eq:problem} (line~\ref{alg:rbt:init}).
Similar to RRT, a random point is sampled (line~\ref{alg:rbt:sample}), and the nearest bubble is identified (line~\ref{alg:rbt:sample}). In doing so, we use the distance to the \emph{boundary} of the existing bubbles:
\begin{equation}\label{eq:dist_bubble}
    d(\bfy, B) = \|\bfy - \bfc\| - r,
\end{equation}
where $\bfc$ and $r$ are the center and radius of $B = (\bfc, r)$ respectively.  

Steering is achieved by setting the new center at the intersection between the boundary of the nearest bubble and the straight line between the random point and the nearest bubble center~(line~\ref{alg:rbt:steer}).
In other words, the new center is steered toward the random sample, up to the \emph{perimeter} of the nearest bubble:
\begin{equation}\label{eq:rbt_steer}
    \bfy_{\snew} = \bfy_{\snear} + r_{\snear} \frac{\bfy_{\srand} - \bfy_{\snear}}{\|\bfy_{\srand} - \bfy_{\snear}\|}.
\end{equation}

The main difference between RBG and RRT is that the random samples $\bfy_{\srand}$ must be outside the union of existing bubbles (line~\ref{alg:rbt:sample}).
This is because, otherwise, the new center after steering~\eqref{eq:rbt_steer} will be contained inside an existing bubble, slowing down the planning progress.
Sampling outside existing bubbles can be achieved with rejection sampling in the simplest case, though more sophisticated and efficient methods may be possible.
With rejection sampling, we found it beneficial to inflate the support of the random samples $\bfy_\srand$, especially in closed obstacles.
This leaves room for random samples to still be drawn from outside existing bubbles even after expansion.
Suitable termination conditions for RBG include checking if a certain number of bubbles have been drawn, or if a given target point is reached. 

\subsection{Expansive Bubble Graph}\label{sec:ebt}

Another pertinent idea to achieve even bubble coverage is to consider the local density of existing samples, as is done in the EST algorithm~\cite{hsu1999est}. We propose the \emph{expansive bubble graph} (EBG) algorithm, which aims to achieve even bubble coverage by limiting overlap with existing bubbles.


\begin{algorithm}[t]
    \caption{Expansive Bubble Graph}\label{alg:ebg}
    \begin{algorithmic}[1]
    \Statex \textbf{Parameters}: No. of directions $N_{\sexplore}$, overlap factor $k_{\soverlap}$, min. radius $r_{\smin}$, footprint radius $\epsilon$, 
    \State $\mathcal{Q} \leftarrow \{ B(\bfy_{s}, d_\Omega(\bfy_s) - \epsilon )\}$ \Comment{Priority queue in descending order of radius}\label{alg:ebg:init}
    \State $\mathcal{B} \leftarrow \{ \}$
    \While{$\mathcal{Q}$ is not empty and termination condition not met}
    \State $B_{\scurrent} = (\bfy_{\scurrent}, r_{\scurrent}) \leftarrow \textsc{pop}(\mathcal{Q})$ 
    \LineComment{Skip if overlap with existing bubbles}
    \If{$\min_{B \in \mathcal{B}} d(\bfy_{\scurrent}, B) < -k_{\soverlap} r_{\scurrent}$} \textbf{continue}\label{alg:ebg:overlap}
    \EndIf
    \State $\mathcal{B} \leftarrow \mathcal{B} \cup \{ B_{current} \}$ \Comment{Otherwise, append and expand}
    \For{$i \in [1, N_{\sexplore}^m]$}
    \State $\bfy_{\snew} \leftarrow \bfy_{\scurrent} + r_{\scurrent} \hat{\mathbf{e}}_{i} $ \Comment{Expand in random or uniform directions.}\label{alg:ebg:expand}
    \If{$r_{new} = d(\bfy_{\snew}) - \epsilon > r_{\smin}$} \Comment{Enqueue if large enough}
    \State $\mathcal{Q} \leftarrow \textsc{push}(\mathcal{Q}, B(\bfy_{\snew}, r_{\snew}) )$\label{alg:ebg:enqueue}
    \EndIf
    \EndFor    
    \EndWhile
    \end{algorithmic}
\end{algorithm}
\begin{figure}[t]
    \centering
    \subfloat[Iteration 2]{
    \includegraphics[width=0.32\textwidth]{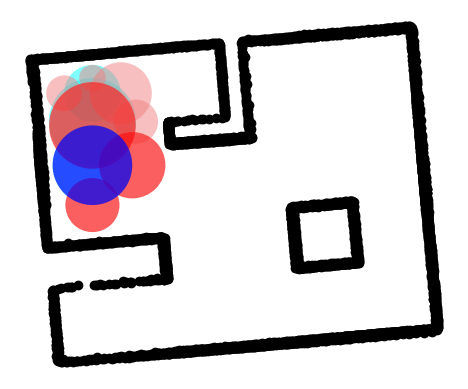}
    } 
    \subfloat[Iteration 3]{
    \includegraphics[width=0.32\textwidth]{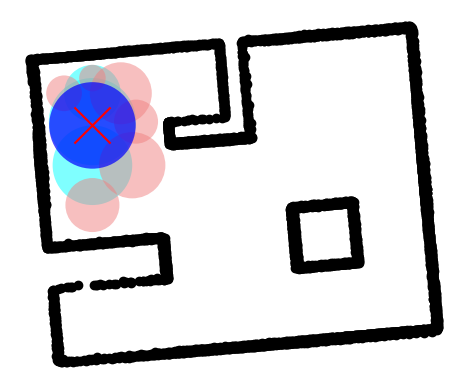}
    }%
    \subfloat[Iteration 190]{
    \includegraphics[width=0.32\textwidth]{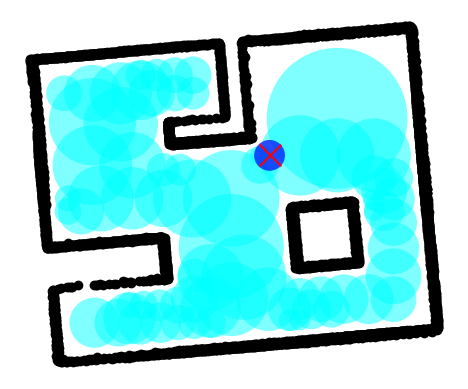}
    }
    \caption{Illustration of EBG with $N_{explore}=4$. a) New bubbles (red) are expanded from current (blue). b) A new bubble from a) is considered, but discarded because of overlap (red cross) c) Repeating this process fills the free space with confirmed bubbles (cyan).}
    \label{fig:ebg_illustration}
\end{figure}

Pseudocode for the EBG algorithm is presented in Alg.~\ref{alg:ebg}, and is illustrated in Fig.~\ref{fig:ebg_illustration}. 
The EBG algorithm expands new bubbles in $N_{\sexplore}$ different directions per dimension, and keeps ones that have limited overlap with existing bubbles and are larger than $r_\smin$. The EBG algorithm starts by initializing a priority queue $\mathcal{Q}$ with a bubble at the robot's starting location~(line~\ref{alg:ebg:init}). 
The queue $\mathcal{Q}$ is in descending order of radii so that larger bubbles appear first.

During iteration, the largest bubble in the queue is popped and added to the cover if there is a limited overlap with existing bubbles~(line~\ref{alg:ebg:overlap}).
Overlap is measured using the ratio $k_{\soverlap}$ of a bubble's own radius $r_{\scurrent}$ to the signed distance of the center to the union of existing bubbles: 
%
$d(\bfy, \cup_{B_e \in \mathcal{B}} B_e) = \min_{B_e \in \mathcal{B}} d(\bfy, B_e)$,
%
where $d(\bfy, B_e)$ is given by~\eqref{eq:dist_bubble}.
\removed{As shown in Fig.~\ref{fig:overlap},} 
Setting $k_{\soverlap}=0$ discards bubbles whose center is on the perimeter of an existing bubble, whereas $k_{\soverlap}=1$ only discards bubbles that are contained in an existing bubble.
\removed{
\begin{figure}[t]
    \centering
    \subfloat[$k=0$]{
    \includegraphics[width=0.25\textwidth]{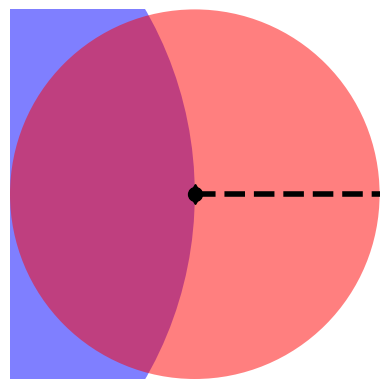}
    }
    \subfloat[$k=0.5$]{
    \includegraphics[width=0.25\textwidth]{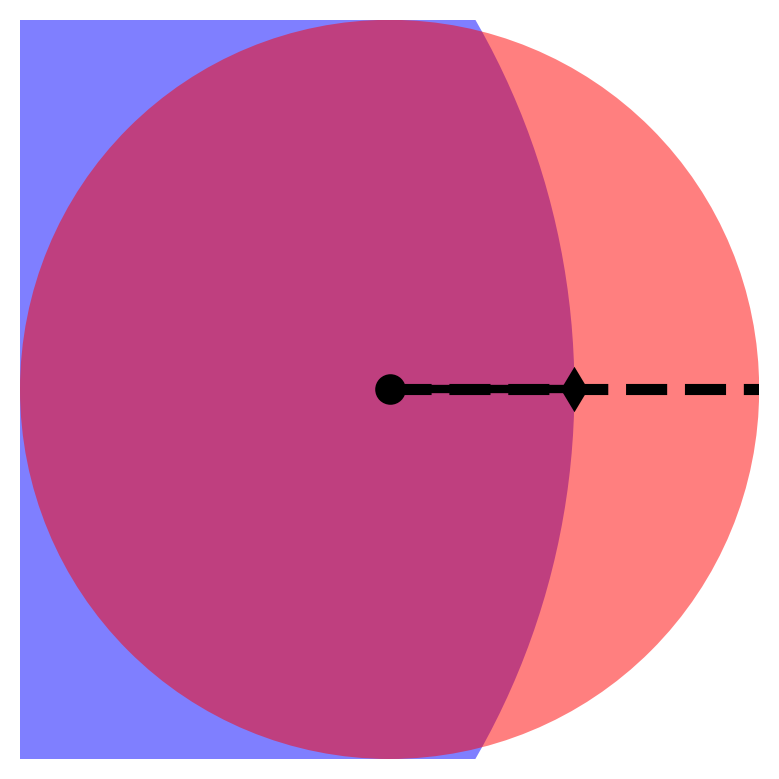}
    }
    \subfloat[$k=1.0$]{
    \includegraphics[width=0.25\textwidth]{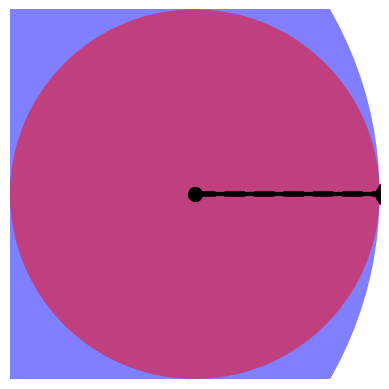}\label{fig:overlap:touching}
    }
    \caption{Visualization of different overlap factors. a) $k_{overlap}=0$ implies bubbles are discarded if the center is on or inside an existing bubble. c) $k_{overlap}=1$ only discards if fully contained. b) At $k_{overlap}=0.5$, new bubbles are discarded if half of its radial line is inside an existing bubble. Dashed line: radius of new bubble. Solid line: distance between new center and boundary of existing bubble.}
    \label{fig:overlap}
\end{figure}
}


Expansion is performed in a total of $N_\sexplore^m$ directions, $\hat{\bfe}_{i}$, which are unit vectors sampled in uniform or random angular increments. 
The expanded bubbles are enqueued for subsequent iterations if the minimum radius requirement $r_{min}$ holds (line~\ref{alg:ebg:enqueue}).
The algorithm terminates if the queue is empty, or if an early termination condition holds.
Suitable early termination conditions for EBG include a certain number of bubbles being reached, or a desired end point being contained in the current bubble.

\section{Planning in Bubble Cover}\label{sec:planning}

The bubble cover $\calB$ generated by the sampling algorithms in Sec.~\ref{sec:sampling} enables a hierarchical discrete-continuous planning method. Using $\calB$ as an approximation of the safe space, the original planning problem \eqref{eq:problem} can be approximately recast as:
\begin{equation}\label{eq:problem_union}
\begin{split}
    \min_{\bfy \in C^{r}} \quad & \int_{0}^{T} c(\bfy(t)) dt, \\
          \text{s.t.} \quad & \bfy(0) = \bfy_{s}, \bfy(T) = \bfy_{g},\\
          & \bfy(t) \in \bigcup_{B \in \mathcal{B}} B, \;\; \forall t \in [0, T],
\end{split}
\end{equation}
where the collision-free constraint has been replaced by containment in the bubble cover.
This conservative reformulation of \eqref{eq:problem} is substantially simpler than the original problem because the nonlinear feasible set has been replaced by a union of simple convex sets. However, the disjunctive containment constraint remains non-convex.


We develop an efficient hierarchical planning method that utilizes the convexity of the bubbles. We first build an intersection graph of bubbles to generate a discrete \emph{bubble path}, followed by using the bubbles along the path as convex constraints to generate a continuous trajectory.

\subsection{Discrete Planning of Bubble Path}
The bubble path should be so chosen to ensure a) feasibility and b) approximate optimality of the continuous trajectory within the bubbles.
Feasibility can be ensured by constructing an intersection graph of the bubble cover.
An intersection graph $\mathcal{G} = (\calB, \calE)$ is an undirected graph where each node is a bubble, and an edge $(i, j) \in \calE$ indicates an overlap between two bubbles $B_i \cap B_j \neq \emptyset$.
Intuitively, where two bubbles are connected by an edge, it is feasible for the robot to continuously move between the two bubbles through the overlap.
\removed{
Alg.~\ref{alg:gob} simply iterates through all unique pairs $B_i$ and $B_j$ to check for overlap.

\begin{algorithm}[t]
    \caption{Construction of Intersection Graph}\label{alg:gob}
    \begin{algorithmic}[1]
    \State $\mathcal{B} \leftarrow \textsc{sample\_bubbles}()$ \Comment{Sample bubbles (see Secs.~\ref{sec:brm},~\ref{sec:ebt},~\ref{sec:rbt})}
    \State $\mathcal{E} \rightarrow \{ \}$ \Comment{Initialize edges as an empty set}
    \For{$B_{i}, B_{j} \in \mathcal{B}$, $i < j$}
    \If{$B_{i} \cap B_{j} \neq \emptyset$} $\mathcal{E} \leftarrow \mathcal{E} \cup \{ (i, j) \}$ \Comment{Add an edge where there is an overlap}
    \EndIf
    \EndFor
    \Return $G = (\mathcal{B}, \mathcal{E})$
    \end{algorithmic}
\end{algorithm}
}

With the intersection graph constructed, approximate optimality of the bubble path can be ensured by considering the `worst-case optimal cost' $\hat{c}(B_{i}, B_{j})$ between two bubbles $B_i$ and $B_j$, defined as: 
\begin{equation}\label{eq:maximum_optimal_cost}
\begin{split}
    \hat{c}(B_{i}, B_{j}) \equiv & \max_{\bfy_{ij}^{s}}  \min_{\bfy_{ij} \in C^{r}} \int_{0}^{T_{ij}} c(\bfy_{ij}(t)) dt\\
    \text{s.t. } \quad & \bfy_{ij}(0) = \bfy_{ij}^{s}, \\
                       & \bfy_{ij}(t) \in B_i, \; \forall t \in [0, T_{ij}), \; \bfy(T_{ij}) \in B_j. 
\end{split}
\end{equation}
Here, the duration $T_{ij}$ of potential trajectory $\bfy_{ij}(t)$ is chosen arbitrarily, e.g., as $T_{ij} = \tfrac{r_i}{V_0}$ for some nominal speed $V_0$ and $r_i$ being the radius of bubble $B_i$.
The intuition for this cost is as follows.
Between bubbles $B_i$ and $B_j$, consider picking an optimal trajectory that is contained within $B_i$ and reaches $B_j$ given a start point $\bfy_{ij}^s \in B_i$, so that the terminal point $\bfy_{ij}(T_{ij})$ is chosen greedily.
Since the terminal point is the start point of the next bubble,~\eqref{eq:maximum_optimal_cost} captures the case when such greedy choice of terminal point is adversarial, so that the terminal point from $B_i$ is the worst start point for $B_j$.
Since the true optimum of~\eqref{eq:problem_union} is a special case with $\bfy_{ij}^s$ and $T_{ij}$ being selected optimally across all bubbles, the sum of maximum optimal cost~\eqref{eq:maximum_optimal_cost} along a path forms an upper bound on the true optimal cost of~\eqref{eq:problem_union}.

When the given cost function is length (i.e., $c(\bfy(t)) = \|\dot{\bfy}(t)\|$), the worst-case optimal cost~\eqref{eq:maximum_optimal_cost} reduces to the single-sided Hausdorff distance:
\begin{equation}\label{eq:hausdorff}
    d_{H}(B_{i}, B_{j}) \equiv \sup_{\bfy_{i} \in B_{i}} d(\bfy_{i}, B_{j}) = | \| \bfc_{i} - \bfc_{j} \| + r_{i} - r_{j}|,
\end{equation}
which holds for two overlapping bubbles $B_i = (\bfc_i, r_i)$ and $B_j = (\bfc_j, r_j)$.


For simplicity, we adopt the Hausdorff distance~\eqref{eq:hausdorff} as an approximation of the worst-case optimal cost~\eqref{eq:maximum_optimal_cost}, with the expectation that shorter length trajectories incur less cost.
With this approximation, the bubble path $\calP$ can be found by solving a graph shortest path problem:
\begin{equation}\label{eq:problem_discrete}
\begin{split}
    \min_{\mathcal{P} = (B_p)} \quad & \sum_{p} d_H(B_{p}, B_{p+1}), \\ 
    \text{s.t.} \quad & (B_{p}, B_{p+1}) \in \mathcal{E}, \\
                      & B_{1} \in \mathcal{B}(\bfy_{s}), B_{|\mathcal{P}|} \in \mathcal{B}(\bfy_{t}),
\end{split}
\end{equation}
where $\mathcal{B}(\bfy_{s})$ and $\mathcal{B}(\bfy_{t})$ denote the bubbles that contain the start position $\bfy_{s}$ and target position $\bfy_{t}$ respectively.
There may be multiple bubbles that contain either start or target position.
Because the Hausdroff distance~\eqref{eq:hausdorff} is non-negative, Dijkstra's algorithm can be used to quickly solve for the shortest path. 



\subsection{Continuous Planning}
With a discrete bubble path $\mathcal{P} = \left\{ B_{p} \right\}_p$ given, we compute a sequence of $|\mathcal{P}|$ continuous segments $\bfy_{p}(s) : [0, T_p] \rightarrow \mathbb{R}^{N}$, one for each bubble in the discrete path $P$. This can be formulated as:
\begin{equation}\label{eq:problem_bezier}
\begin{split}
    \min_{\bfy_{1}(t), \ldots \bfy_{|\mathcal{P}|}(t)} \quad  & \sum_{p} \int_{0}^{T_p} c(\bfy_{p}(t)) dt, \\
          \text{s.t. (bubble containment)} \quad & \bfy_{p}(t) \in B_{p}, \quad \forall p, t, \\
          \text{ (start/end points)} \quad & \bfy_{1}(0) = \bfy_{s}, \bfy_{|\mathcal{P}|}(1) = \bfy_{g},\\
          \text{ (continuous derivatives)} \quad& \bfy_{p}^{(d)}(T_p) = \bfy_{p+1}^{(d)}(0), \quad \forall d \in [0, r], \forall p.
\end{split}
\end{equation}

Conveniently, the continuous problem~\eqref{eq:problem_bezier} can be solved as a convex program by parameterizing the trajectories $\bfy_{p}(t)$ as \emph{Bezier curves} given by:
\begin{equation}
    \bfy_{p}(t) = \sum_{k=0}^{K} b_{k}\left(\tfrac{t}{T_p}\right) \mathbf{b}_{k}^{p},
\end{equation}
where $b_{k}(s) = \binom{K}{k} s^{k}(1-s)^{K-k}$ are the Bernstein basis polynomials, $\mathbf{b}_{k}^{p} \in \mathbb{R}^{m}$ are the \emph{control points}, $T_p$ is the duration of each trajectory, and $K$ is the order of the Bezier curve. As noted in~\cite{maruci2023gcs_planning}, Bezier curves have several useful properties that allow parameterizing \eqref{eq:problem_bezier} as a convex program of control points. The start and end of each Bezier curve are given by $\bfy_p(0) = \bfb_0^p$, $\bfy_p(T_p) = \bfb_K^p$, and the derivatives are another Bezier curve with control points given by the finite difference $\Delta[\bfb_{k}^p] = \bfb_{k+1}^p - \bfb_k^p$ of control points as
    $\dot{\bfy}_{p}(t)= \sum_{k}^{K-1} b_k(\tfrac{t}{T_p}) \tfrac{K}{T_p} \Delta[\bfb_{k}^{p}]$.
These two properties show that the derivative continuity and start/end points constraints in~\eqref{eq:problem_bezier} are affine in the control points $\bfb_{k}^{p}$.

More importantly, each curve $\bfy_{p}$ is entirely contained in the convex hull of control points\removed{(i.e. $\bfy_{p}(t) \in \text{conv}\{ \mathbf{b}_{k}^{p} \}_{k=1}^{K}$)}.
Therefore, a \emph{sufficient} relaxation of the bubble containment constraint is to ensure that all control points are contained in respective bubbles, i.e., $\mathbf{b}_{k}^{p} \in B_{p}$, $\forall k, p$. 
Moreover, the cost function remains a convex function of control points since the trajectory is an affine function of the control points at each $t$, and can be upper bounded as $\int_{0}^{T_p} c(\bfy_p(t)) dt \leq \tfrac{1}{K + 1} \sum_{k} c(\bfb_k^p)$ \cite{maruci2023gcs_planning}.  
Furthermore, common cost functions, such as the time integral of squared norm of $n$-th order derivatives of Bezier curves can be represented as a positive semidefinite quadratic form of $\mathbf{b}_{k}^{p}$ \cite{mellinger2011differential}.

Combining these properties, the continuous problem~\eqref{eq:problem_bezier} can be written in terms of control points $\bfb = \{\bfb_{k}^{p}\}$ in the following form:
\begin{equation}
\begin{split}
    & \min_{\bfb } \quad  \sum_{p, k} c(\bfb_{k}^{p}), \\
    \text{s.t. (bubble containment)} \quad & \bfb_{k}^p \in B_{p}, \quad \forall p, k, \\
    \text{ (start/end points)} \quad & \bfb_0^{0} = \bfy_{s}, \bfb_K^{|\calP|} = \bfy_t, \\
    \text{ (continuous derivatives)} \quad & (\Delta)^{d}[\bfb_{K-d}^{p}] = (\Delta)^{d}[\bfb_0^{p+1}], \quad \forall d \in [0, r], \forall p, 
\end{split}    
\end{equation}
where $\Delta^{d}$ is the $d$-th finite difference. The bubble containment constraint is quadratic, while the other constraints are affine. As the cost is convex, the overall problem is a convex program. In the special case when the cost is the squared norm of $d$-th derivatives, the cost is quadratic, and the overall problem is a quadratically-constrained quadratic program.





\subsection{Extension to Unknown Environments}

\begin{figure}[t]
    \centering
    \subfloat{\includegraphics[width=0.32\linewidth]{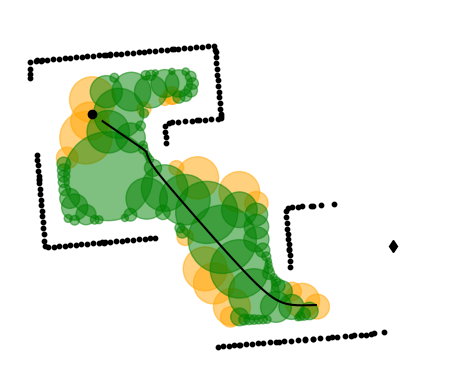}}
    \subfloat{\includegraphics[width=0.32\linewidth]{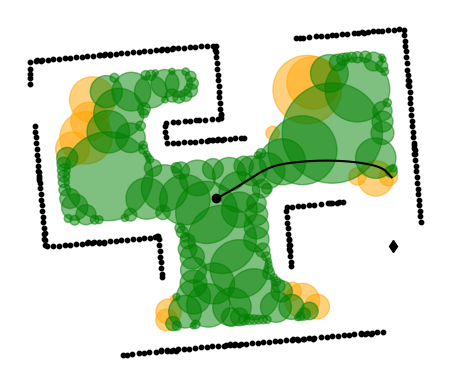}}
    \subfloat{\includegraphics[width=0.32\linewidth]{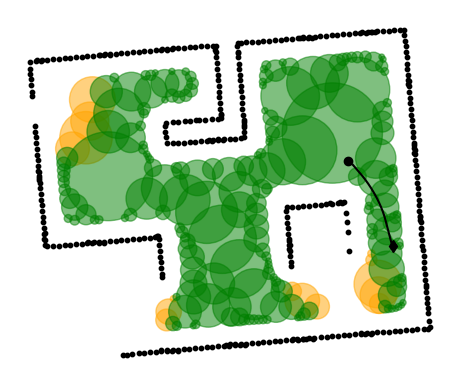}}
    \caption{An illustration of EBG modified for static unknown environments. Before the goal (black diamond) is in fully expanded bubbles (green), the planned trajectory (black solid line) points to the closest fully expanded bubble. As frontier bubbles (yellow) become visible, more fully expanded bubbles (green) appear.}
    \label{fig:unknown_env}
\end{figure}

An important aspect of motion planning algorithms is operation in an unknown environment.
We briefly outline an approach to planning in static unknown environments inspired by the concept of frontiers in occupancy grids \cite{yamauchi1997frontier}. The main idea is to verify whether a safe bubble had been fully \emph{visible} from the robot at a particular pose.
Those that had not been fully visible are the \emph{frontier} bubbles.
The queue $\calQ$ in Alg.~\ref{alg:ebg} is modified, so that frontier bubbles are skipped, and only those fully visible are expanded, until only frontier bubbles remain.
Upon receiving new sensor data, the visibility information is updated, and the expansion loop is repeated.
To plan a path towards a goal in this incomplete cover, we simply pick the closest bubble to the goal, with an additional terminal cost of distance to the goal. 
An example result is illustrated in Fig.~\ref{fig:unknown_env}.

\section{Evaluation}
\begin{figure}[t]
    \centering
    \subfloat[Gazebo Room\label{fig:environments:gazebo}]{\includegraphics[width=0.3\linewidth]{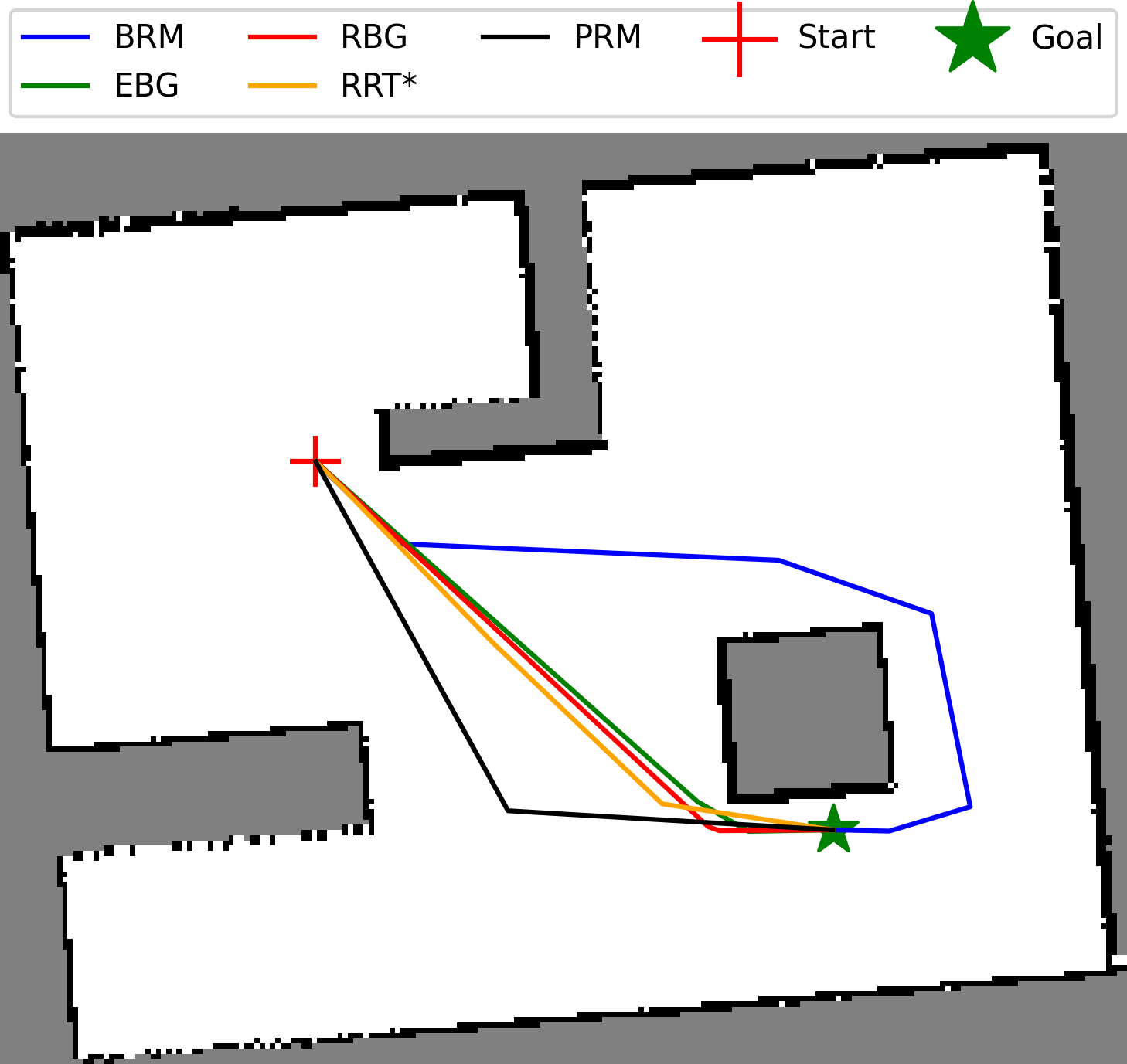}}%
    \subfloat[House Expo\label{fig:environments:house_expo}]{\includegraphics[width=0.3\linewidth]{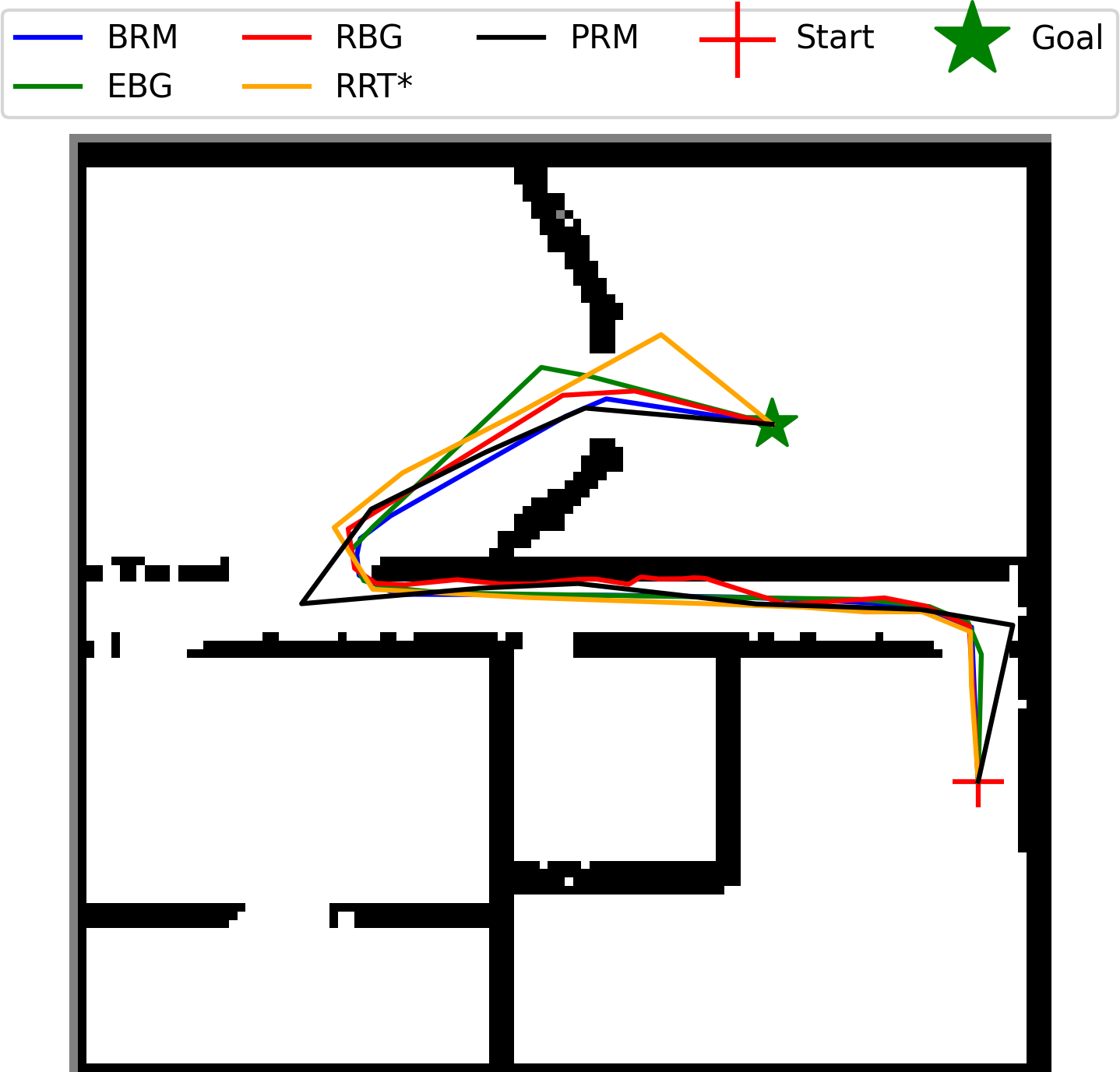}} \\
    \subfloat[Replica Hotel\label{fig:environments:replica}]
    {\includegraphics[width=0.6\linewidth]{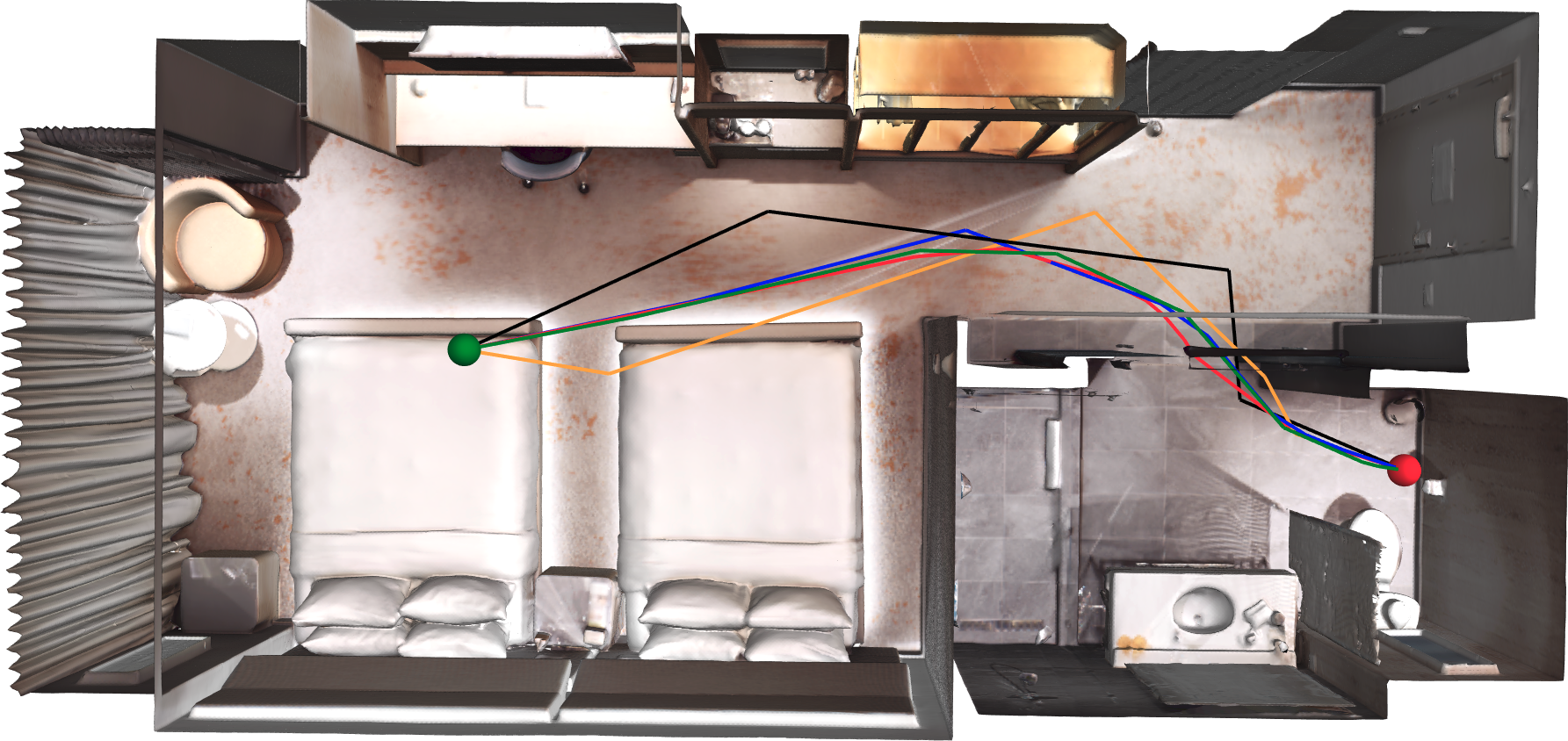}}
    \caption{Qualitative comparison of shortest-distance paths computed by different algorithms in three benchmarks. Paths shown correspond with the lowest computational effort recorded.} \label{fig:environments}
    \vspace{-3ex}
\end{figure}


We evaluate the performance of the safe bubble cover algorithms in three benchmark environments, shown in Fig.~\ref{fig:environments}. There are two 2D environments, namely the Gazebo Room \cite{wu2023loggpis} and the House Expo \cite{li2019houseexpo}, which are representative of indoor environments. The House Expo, in particular, features the challenge of a relatively narrow corridor.
We also consider a 3D environment, the Replica Hotel \cite{replica19arxiv}, which is widely used as a navigation benchmark, and represents an indoor environment cluttered with objects and a narrow entrance to the bathroom area.
For all environments, the distance function is built from simulated LiDAR data using the Log-GPIS algorithm \cite{wu2023loggpis}, which allows querying distance values at arbitrary continuous query points.

\subsection{Planning Shortest Distance Paths}\label{sec:shortest_paths}
We first compare the planning performance of BRM, RBG, and EBG against two classical sampling-based planning algorithms, namely PRM$^*$ \cite{kavraki1996prm,karaman2011optimal} and RRT$^*$ \cite{lavalle2001rrt,karaman2011optimal}, in planning shortest distance paths. PRM$^*$ and RRT$*$ perform collision checking by querying the distance field along each edge with a resolution of $0.05$ m, which was chosen empirically to ensure correct collision checking.

For each of the three environments, 100 random start/goal pairs were chosen.
For each start/goal pair, we repeated each planning algorithm with five different random seeds.
For each run, we evaluated the number of SDF query positions, success rate of finding a path, and the path cost of successful runs.

\begin{figure}[t]
    \centering
    \subfloat[Gazebo Room\label{fig:iter_query:gazebo}]{\includegraphics[width=0.32\textwidth]{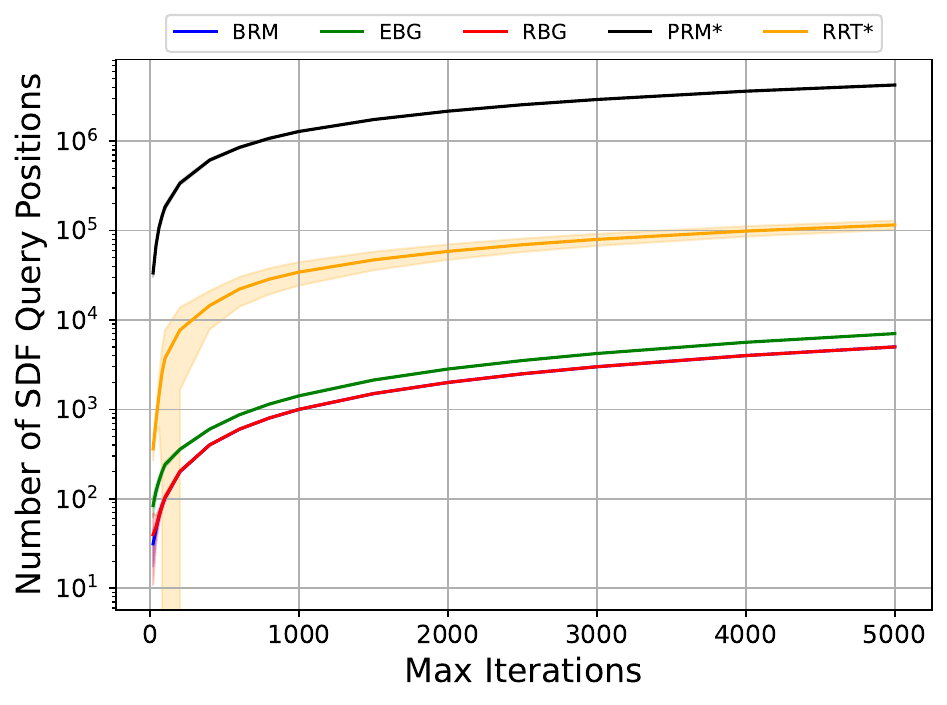}}
    \subfloat[House Expo\label{fig:iter_query:house_expo}]{\includegraphics[width=0.32\textwidth]{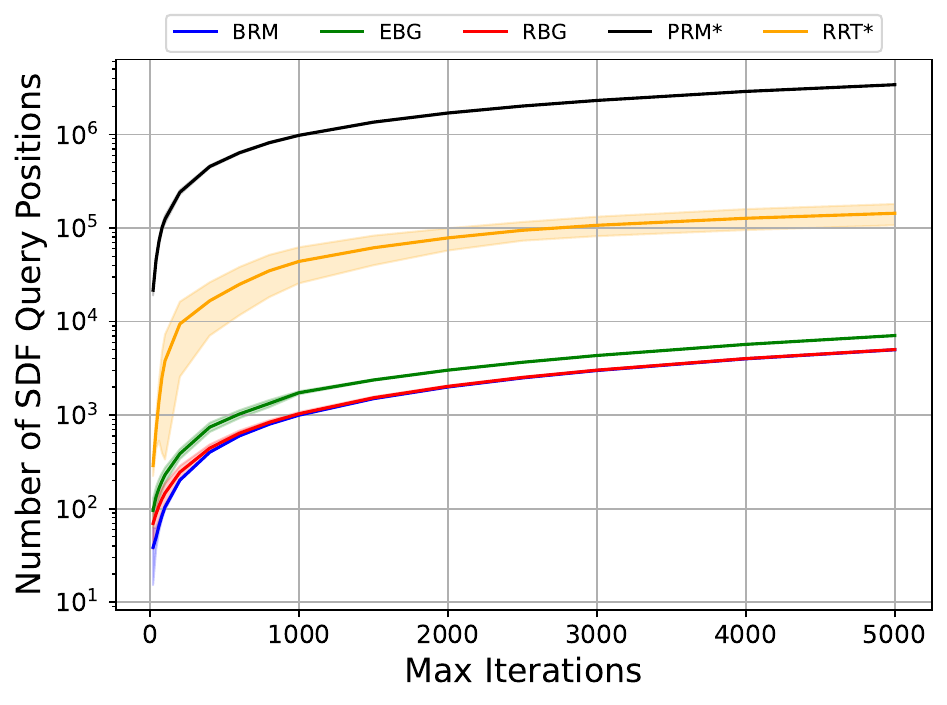}}
    \subfloat[Replica Hotel\label{fig:iter_query:replica}]{\includegraphics[width=0.32\textwidth]{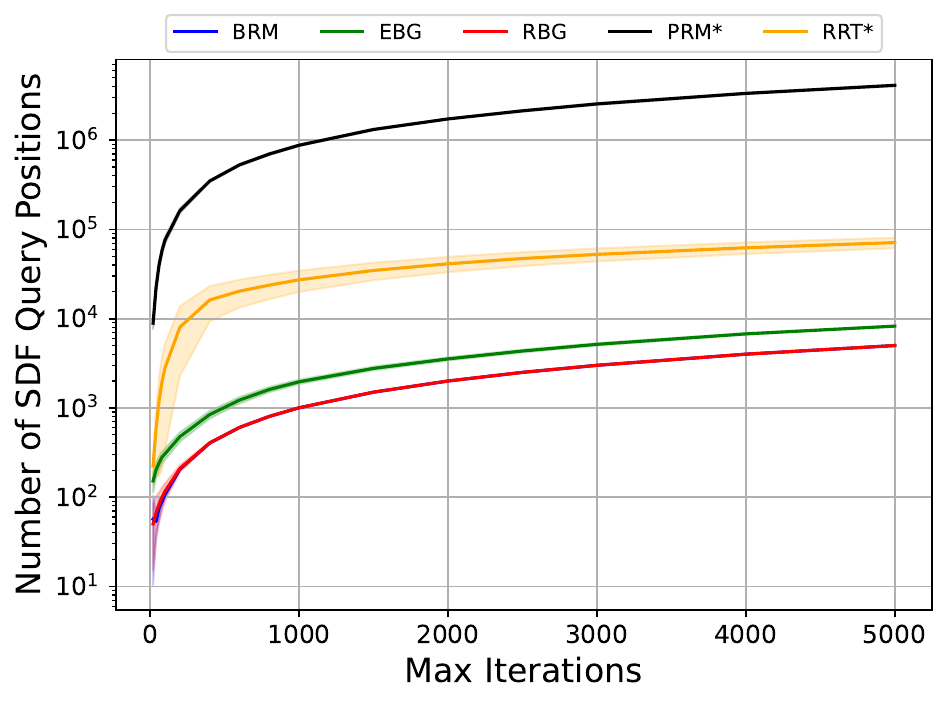}}    
    \caption{Comparison of number of SDF query positions over maximum allowed iterations. The number of SDF queries signifies the computational effort required. Bubble-based methods incur 1-2 orders of magnitude less computational effort than conventional baselines. The results are averaged over 100 random start/end pairs, each with 5 different random seeds.}
    \label{fig:iter_query}
\end{figure}

The results are shown in Fig.~\ref{fig:iter_query}-\ref{fig:eval_num_test_positions}.
Fig.~\ref{fig:iter_query} shows the number of unique SDF query positions over the maximum number of iterations or samples, which is representative of the computation time in an optimized implementation.
It can be seen that the classical sampling-based algorithms (RRT$^*$ and PRM$^*$) make 1 - 3 orders of magnitude more SDF queries than the bubble cover algorithms (BRM, RBG and EBG).
This shows that the bubble cover algorithms are 1 - 3 orders of magnitude more computationally efficient per iteration.

\begin{figure}[b]
    \centering
    \subfloat[Gazebo Room\label{fig:query_to_succ_rate:gazebo}]{\includegraphics[width=0.32\textwidth]{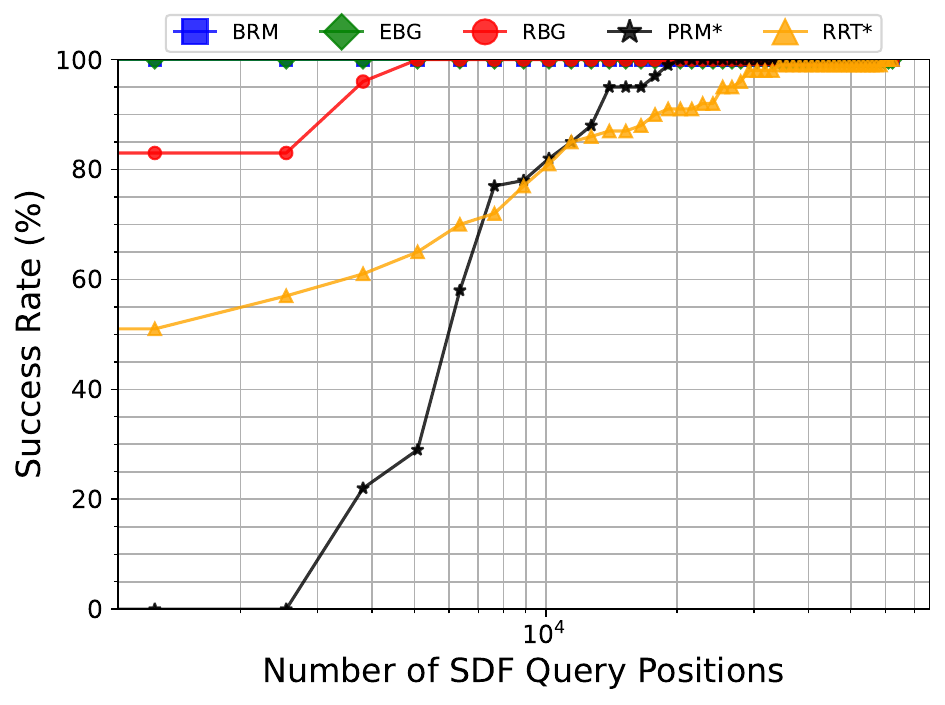}}
    \subfloat[House Expo\label{fig:query_to_succ_rate:house_expo}]{\includegraphics[width=0.32\textwidth]{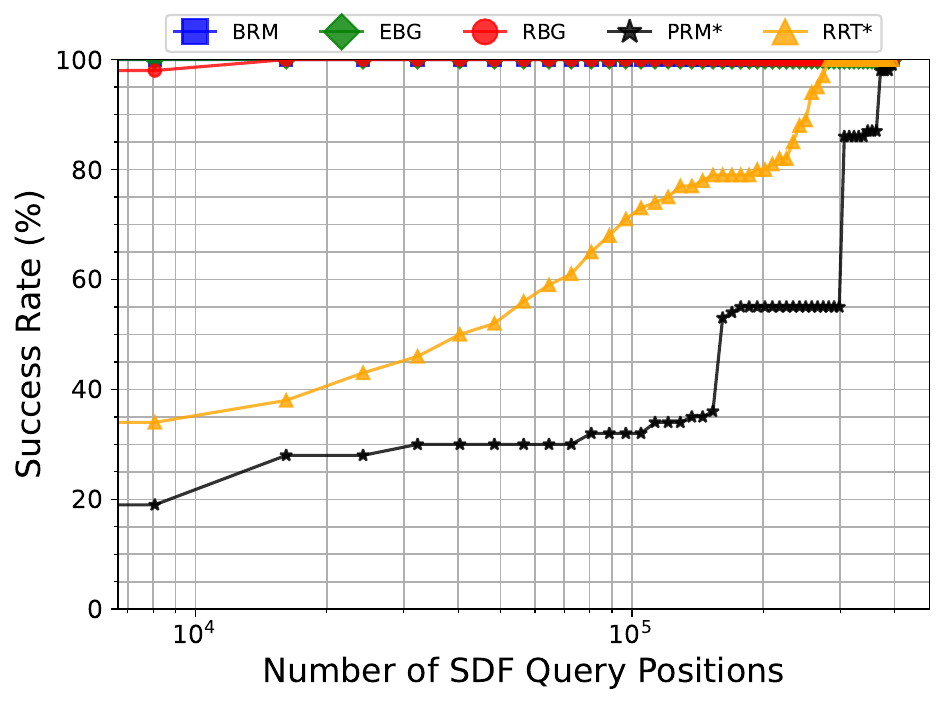}}
    \subfloat[Replica Hotel\label{fig:query_to_succ_rate:replica}]{\includegraphics[width=0.32\textwidth]{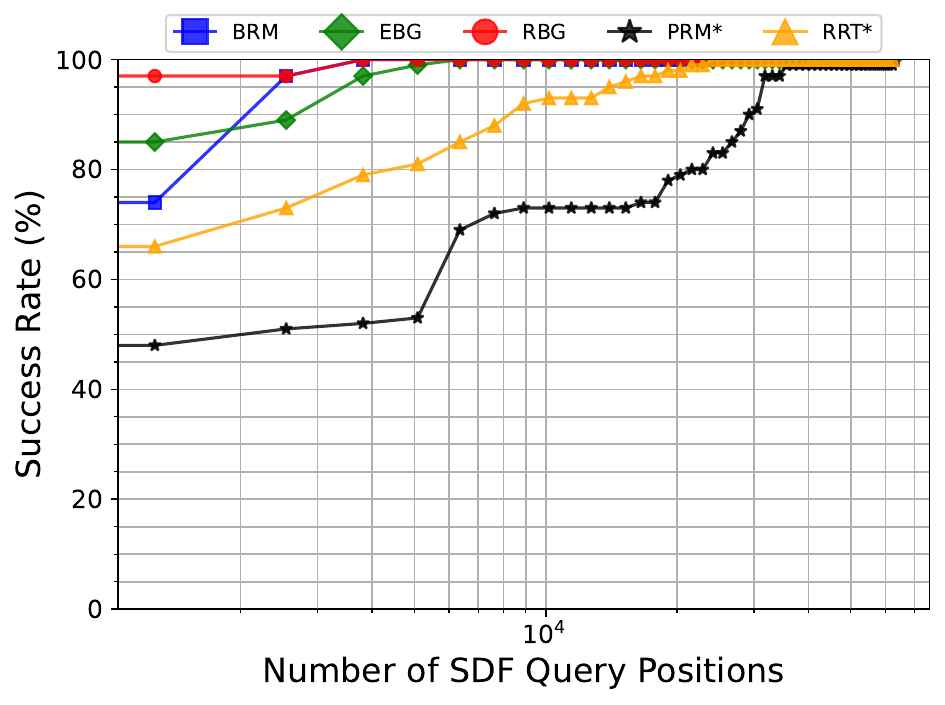}}    
    \caption{Comparison of success rate over number of SDF querry positions. Bubble-based methods reach 90\% success rate with at least four times less computational effort (in Replica Hotel). The evaluation is performed over 100 random start/end pairs, each with 5 different random seeds.}
    \label{fig:query_to_succ_rate}
\end{figure}

In Figs.~\ref{fig:query_to_succ_rate} and~\ref{fig:eval_num_test_positions}, we compare the success rate and path cost relative to computation effort represented by the number of SDF query positions, because each iteration incurs varying computational effort between different algorithms. 
For comparison, we normalize the path cost by the worst run in Fig.~\ref{fig:eval_num_test_positions}, since start/goal pairs vary.
In Fig.~\ref{fig:query_to_succ_rate}, it can be seen that RRT$^*$ and PRM$^*$ require at least four times more SDF queries than bubble-based algorithms to achieve 90\% success rate across all environments.
The distribution of normalized path cost and number of SDF query in Fig.~\ref{fig:eval_num_test_positions} shows that PRM$^*$ generally produces the highest cost paths.
In comparison, BRM, RBG and EBG are distributed near the bottom left corner, with up to 10-fold reduction in cost and computation simultaneously, with RBG and EBG showing advantage over BRM in 3D.
The qualitative comparison in Fig.~\ref{fig:environments} shows that, even with limited compute, the bubble-based methods produce shorter paths that cut corners.
This is because the bubble-based methods naturally incorporate continuous trajectory optimization.

\subsection{Comparison of Bubble Sampling Algorithms}
To evaluate the effectiveness of the proposed sampling methods, we task BRM, RBG and EBG to plan for a smooth, minimum snap trajectory \cite{mellinger2011differential}, and compare the cost and the area of reachable space of the safe bubble covers. The planning is done in the same setting as Sec.~\ref{sec:shortest_paths}, except for the cost function for continuous planning. The resulting trajectories can be used to control a quadrotor as illustrated in Fig.~\ref{fig:teaser}, by recovering the thrust and angular velocity inputs using differential flatness methods from \cite{morrell2018differential,liu2018quadrotor}.

To compute the reachable area, we generate safe bubble covers from 200 randomly initialized seed locations in the free space in the benchmark environments. 
We record the area of each bubble cover every 50 iterations, approximated using the Monte Carlo method with 100000 random samples from free space.
Since BRM is not iterative, we consider a varying number of samples instead of iterations.
Moreover, because BRM does not guarantee the bubble cover to be connected, we only consider the area of the connected component from the same starting location as RBG and EBG, to faithfully represent the utility in a planning scenario.
\begin{figure}[t]
    \centering
    \subfloat[Gazebo Room\label{fig:eval_num_test_positions:gazebo}]{\includegraphics[width=0.32\textwidth]{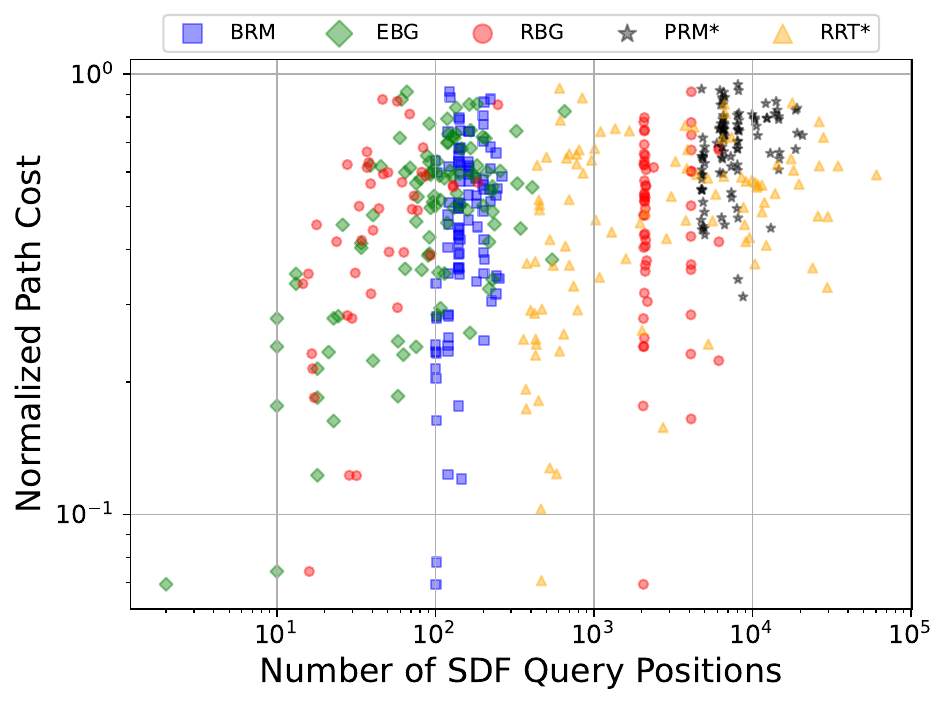}}
    \subfloat[House Expo\label{fig:eval_num_test_positions:house_expo}]{\includegraphics[width=0.32\textwidth]{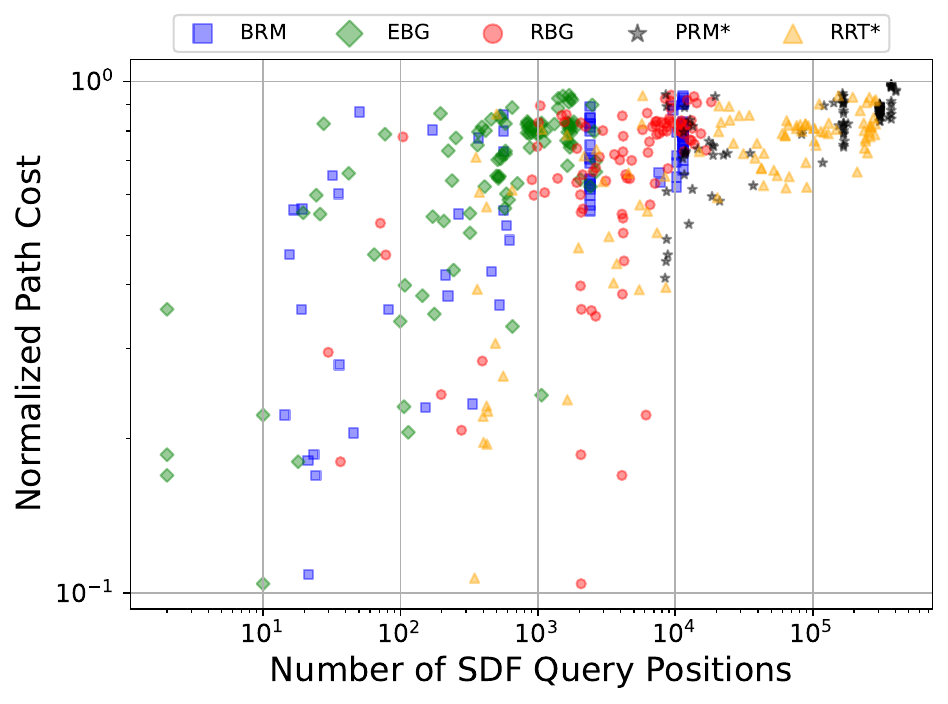}}
    \subfloat[Replica Hotel\label{fig:eval_num_test_positions:replica}]{\includegraphics[width=0.32\textwidth]{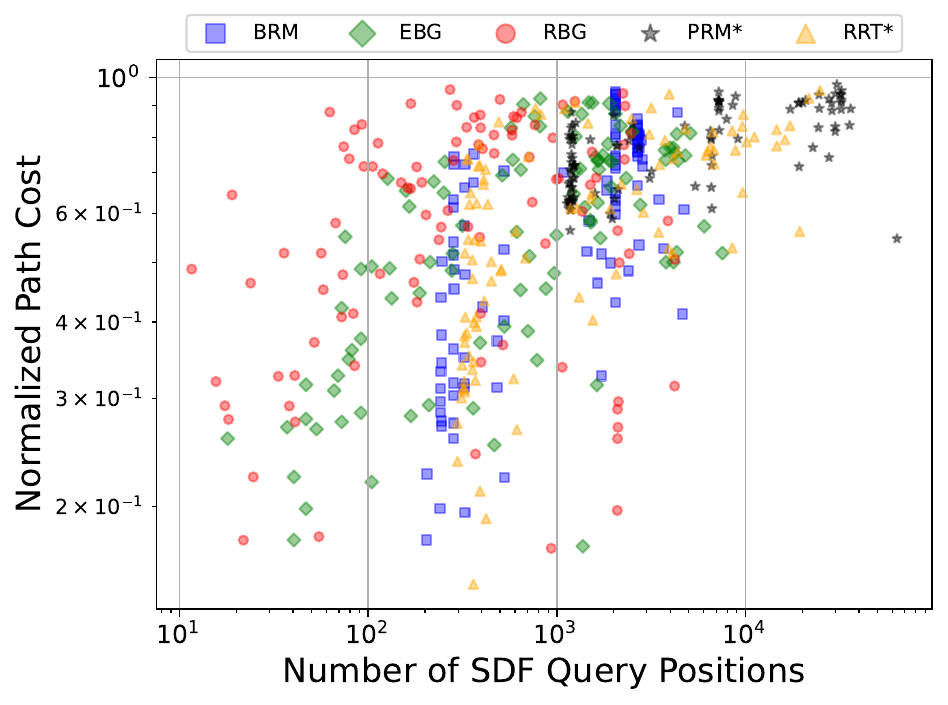}}        
    \caption{Comparison of normalized trajectory length relative to worst, over number of SDF query positions. The bubble-based methods return shorter paths in less time. The results show 100 random start/goal pairs, each with 5 different random seeds.}
    \label{fig:eval_num_test_positions}
\end{figure}
\begin{figure}[t]
    \centering
    \subfloat[Gazebo Room\label{fig:eval_reachable:gazebo}]{\includegraphics[width=0.3\linewidth]{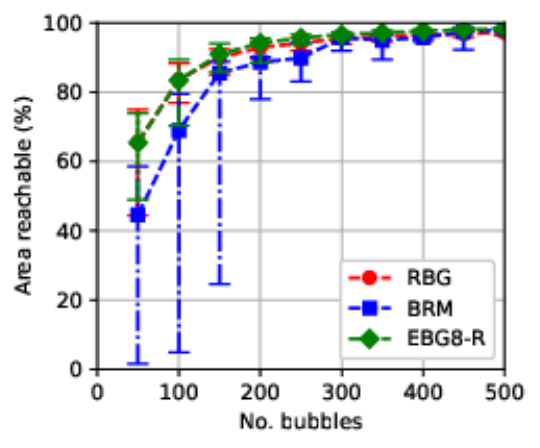}
    }
    \subfloat[House Expo\label{fig:eval_reachable:house_expo}]{\includegraphics[width=0.3\linewidth]{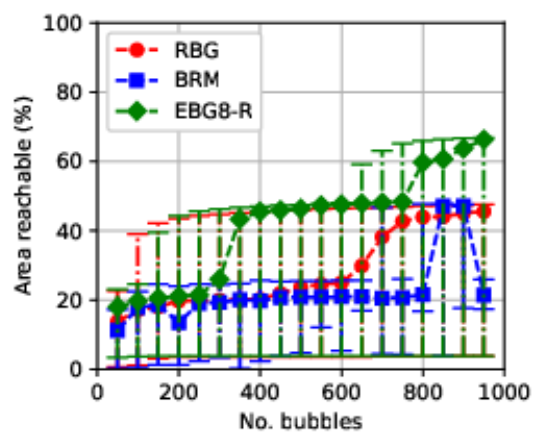}
    }
    \subfloat[Replica Hotel\label{fig:eval_reachable:replica}]{\includegraphics[width=0.3\linewidth]{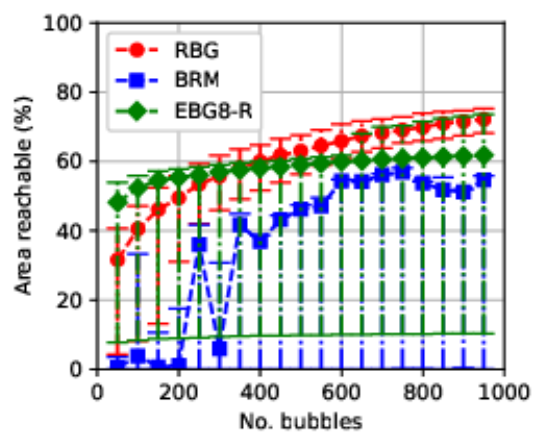}
    }    
    \caption{Comparison of reachable area over iterations. Trend lines: median, error bars: 10\% and 90\% quantiles. RBG and EBG perform almost equally well in coverage (except in House Expo), while BRM performs worse with large variability. EBG used with 8 random directions.} 
    \label{fig:eval_reachable}
\end{figure}
The results are shown in Fig.~\ref{fig:eval_reachable}.
It can be seen that in the Gazebo Room (Fig.~\ref{fig:environments:gazebo}), all methods eventually cover nearly 100\% of the free space (Fig.~\ref{fig:eval_reachable:gazebo}).
RBG and EBG cover the space faster than BRM, which is attributed to RBG and EBG having respective means to promote even spatial coverage.
The 10\% and 90\% quantiles of RBG and EBG are also narrower than BRM, which shows that RBG and EBG perform more reliably than BRM in covering the safe space.

A similar pattern is observed in the House Expo environment in Fig.~\ref{fig:eval_reachable:house_expo}, which is the most challenging.
Only up to 75\% of the environment is covered in the best case, due to the narrow corridor.
In this setting, RBG performs similarly to BRM in median, albeit with a higher 10\% quantile.
EBG performs the best, because it continues to make progress in the narrow corridor by expanding from current bubbles, whereas RBG and BRM rely on random samples.

\begin{figure}[t]
    \centering
    \subfloat[Gazebo Room]{\includegraphics[width=0.3\linewidth]{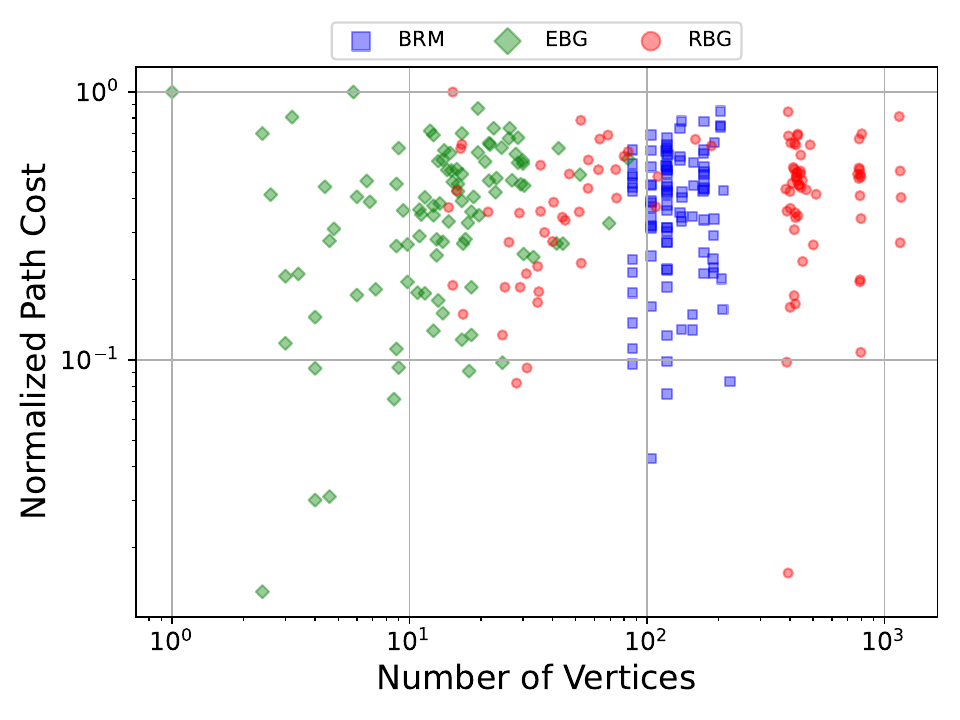}
    }
    \subfloat[House Expo]{\includegraphics[width=0.3\linewidth]{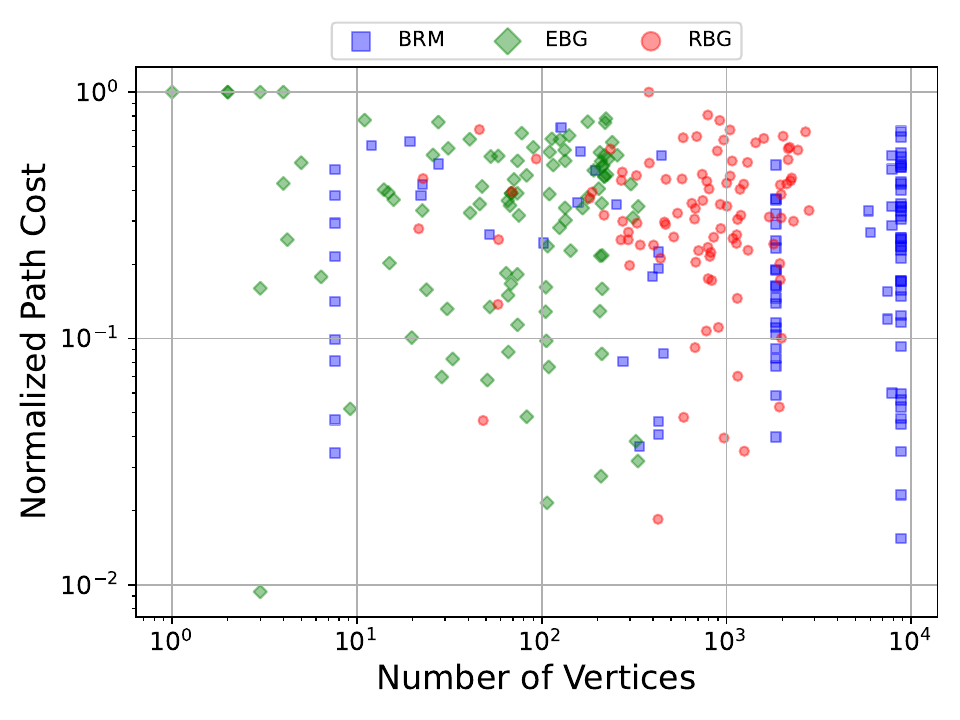}
    }
    \subfloat[Replica Hotel]{\includegraphics[width=0.3\linewidth]{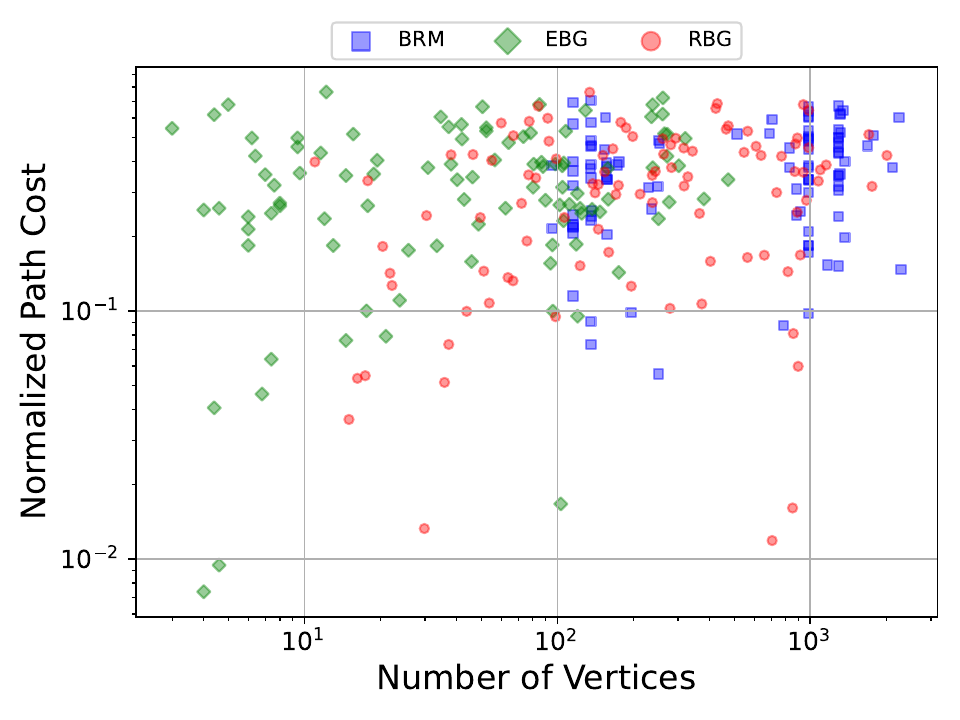}
    } \\    
    \caption{Comparison of minimum snap trajectory cost over number of bubbles. Bubbles from EBG generally find the lowest-cost paths with fewer bubbles, followed by RBG and BRM.}
    \label{fig:eval_continuous_cost}
\end{figure}

The Replica Hotel (Fig.~\ref{fig:environments:replica}) is also challenging, with all methods covering up to 90\% of the free space (Fig.~\ref{fig:eval_reachable:replica}).
It can be seen that RBG and EBG generally cover safe space faster than BRM at all number of iterations in the Replica Hotel environment. This is consistent with the observation from the Gazebo Room environment, with a greater gap.
The greater gap between BRM and other methods shows that ensuring even spatial coverage is more important with greater environment complexity and higher number of dimensions.
Moreover, RBG outperforms EBG in median reachable area except at the very initial and final stages, with consistently narrower quantiles.
We thus conclude that a) RBG exhibits better spatial coverage than EBG in 3D, and hypothesize that RBG will scale better to higher dimensions than EBG, and b) EBG is more suited for environments with narrow corridors. 

The cost distribution plot in Fig.~\ref{fig:eval_continuous_cost} suggests that, unlike for shortest length objective, EBG generally performs the best, followed by RBG and BRM.
We attribute this to two factors: a) there is a disparity between the discrete planning objective (Hausdorff distance as an upper bound on the length) and the continuous planning objective (minimum snap); and b) EBG has more overlapping areas than RBG, because RBG only expands outwards.
In this case, continuous planning can exploit the overlaps in EBG to better resolve the disparity in cost functions used. 

\section{Conclusion} 
\label{sec:conclusion}

We presented hybrid discrete-continuous sampling-based planning methods based on the idea of safe bubbles. We showed that safe bubbles can be defined for any Lipschitz-continuous safety constraint, with distance function as a special case.
We introduced three sampling algorithms for safe bubble cover construction, namely BRM, RBG and EBG, drawing inspirations from PRM, RRT, and EST respectively, and developed a hierarchical method for planning continuous trajectories in the safe bubble cover.
Our evaluations show that bubble-based methods yield trajectories with an order of magnitude lower cost, while being an order of magnitude more computationally efficient owing to the lack of explicit collision checking.

We anticipate that our results will lead to a new class of sampling-based planning algorithms for implicit representations that side-step collision checking and efficiently generate continuous trajectories. 
We hope that the proposed sampling techniques inspire other sampling methods for safe bubbles drawing upon decades of research in sampling-based planning.
We plan to support this direction through theoretical analysis of the bubble cover methods and applications to multi-rigid-body robots in future work.




\subsubsection*{Acknowledgement}
This work was supported by the Ministry of Trade, Industry, and Energy (MOTIE), Korea, under the Strategic Technology Development Program supervised by the Korea Institute for Advancement of Technology (KIAT) [Grant No. P0026052].
Cedric and Teresa are supported by the Australian Research Council Discovery Project under Grant DP210101336.

\end{document}